\documentclass[10pt,twocolumn,letterpaper]{article}

\usepackage{wacv}              

\usepackage{tabularx}
\usepackage{array}
\usepackage{comment}
\usepackage{algorithm}
\usepackage{algpseudocode}

\definecolor{wacvblue}{rgb}{0.21,0.49,0.74}
\usepackage[pagebackref,breaklinks,colorlinks,allcolors=wacvblue]{hyperref}
\usepackage{xcolor}

\usepackage{amsthm}   
\newtheorem{fedranprop}{Proposition}
\newtheorem{fedranthm}{Theorem}
\newtheorem{fedrancor}{Corollary}

\usepackage{balance}
\usepackage{tikz}
\usetikzlibrary{arrows.meta,positioning,fit,calc,shapes.geometric}

\def\wacvPaperID{2870} 
\def\confName{WACV}
\def\confYear{2027}

\title{FedRAN: Accurate and Efficient Analytic Federated Continual Learning}

\author{
Jebacyril Arockiaraj$^{1}$ \qquad Dhruv Parikh$^{1}$\\
Jayashree Adivarahan$^{1}$ \qquad Rajgopal Kannan$^{2}$ \qquad Viktor Prasanna$^{1}$\\[0.5em]
$^{1}$University of Southern California
\qquad
$^{2}$DEVCOM Army Research Office\\
{\tt\small
\{arockiar,dhruvash,adivarah,prasanna\}@usc.edu
\qquad
rajgopal.kannan.civ@army.mil}
}

\begin{document}
\maketitle

\begin{abstract}

Federated continual learning (FCL) enables distributed clients to learn from evolving data streams while retaining knowledge from previously observed tasks. Existing optimization-based FCL methods often rely on repeated local training, incurring communication overhead and feature drift across clients. Analytic alternatives avoid iterative training, but high-dimensional random features used to improve accuracy require communicating a Gram matrix whose cost grows quadratically with feature dimension. Communication-efficient alternatives reduce this cost by statistically reconstructing the Gram matrix from first-order summaries. We propose FedRAN, a low-rank analytic FCL framework that instead compresses the required second-order structure directly through spectral summaries that are merged across clients and tasks. The retained rank explicitly controls the tradeoff between approximation fidelity and communication, reducing second-order communication from quadratic to linear in feature dimension for a fixed rank. FedRAN further supports label scarcity through prototype-based pseudo-labeling. Across CIFAR-100, ImageNet-R, and VTAB, FedRAN improves average accuracy by up to 4.8 percentage points over the strongest baseline, uses 30.6--121.8$\times$ less per-client communication, and is 190.3$\times$ faster than optimization-based baselines.

\end{abstract}

\section{Introduction}

\begin{figure*}[t!]
    \centering
    \includegraphics[width=0.8\textwidth]{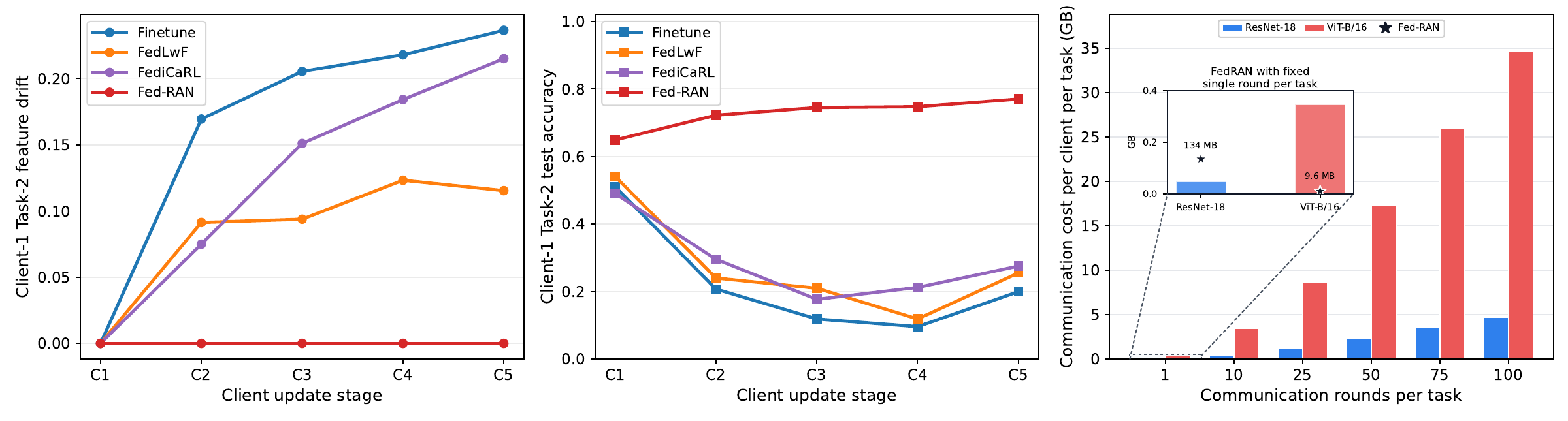}
    \caption{
    Motivation for efficient FCL.
    Left and middle: sequential non-IID client updates of trainable baselines increase feature drift on a fixed Client-1 Task-2 reference set and reduce test accuracy, while FedRAN keeps the representation fixed.
    Right: per-client communication for iterative methods grows with communication rounds and exchanged model size, whereas FedRAN uses a fixed one-shot statistic upload.
    }
    \label{fig:motivation}
\end{figure*}

Federated learning (FL) enables multiple clients to collaboratively learn a model while keeping raw data local~\cite{FedAvg, FL_survey}. In many practical deployments, data continue to arrive after deployment, and local client distributions evolve as new classes, domains, or sensing conditions appear~\cite{jothimurugesan2023federated, target}. Continual learning (CL) studies how models can learn from such non-stationary streams without forgetting prior tasks~\cite{fedEWC,schwarz2018progress,de2021continual}. Federated continual learning (FCL) combines these two axes: clients learn from private sequential task streams while a server maintains a global model over all classes observed so far~\cite{FedWeiT, FCL_survey}. A central challenge in FCL is that practical deployments are constrained not only by privacy and catastrophic forgetting, but also by the resources required to keep learning. Recent benchmarking of FCL shows that existing methods often assume an unrestricted training budget and degrade substantially when memory, computation, communication rounds, and label rate are limited~\cite{li2026resource}. 

Existing FCL methods are largely optimization-centric, relying on iterative local training 
~\cite{FedWeiT,dong2023no,qi2023better,target,bakman2023federated,FedCPrompt,Pilora}. This creates two challenges in efficient FCL deployment. \emph{First}, iterative training repeatedly exchanges model parameters or parameter updates between the server and clients: each communication round distributes the current global model, performs local optimization, and returns updated client parameters to the server. Communication therefore grows with the number of rounds and exchanged parameter size, while client computation grows with repeated local training. \emph{Second}, when the global model is updated on non-IID client data, local updates can induce feature drift across clients, making the same client's data map to different representations after other clients update the global model~\cite{venkatesha2022addressing}. As shown in Fig.~\ref{fig:motivation}, feature drift increases across sequential client updates and coincides with lower test accuracy, while communication from repeated model exchange grows with the number of training rounds.


Analytic approaches offer a different route: they keep the feature extractor fixed and update the classifier from additive feature statistics in closed form~\cite{fed3r2024,afcl2025,stsa}. High-dimensional random features can improve feature separability and accuracy~\cite{mcdonnell2023ranpac,loranpac2025}, but exact aggregation requires communicating their Gram matrix, whose size grows quadratically with feature dimension. Communication-efficient alternatives reduce this cost by reconstructing the Gram matrix from first-order class statistics~\cite{stsa,fedcof2025}. These summaries are compact, but they do not directly preserve the feature correlations used by the analytic classifier; the required second-order structure must therefore be estimated. FedRAN instead directly compresses this second-order structure, retaining its dominant directions in a low-rank summary. This reduces second-order communication from quadratic to linear in feature dimension for a fixed rank, while allowing the retained rank to control how much information is preserved.

FedRAN realizes this idea by maintaining a compact global spectral state over the distributed task stream. Each client keeps the backbone fixed and uploads two statistics: an exact label-feature statistic and a truncated-SVD summary of its local random-feature Gram matrix. The server sums the label-feature statistics and merges the spectral summaries through a two-level QR--SVD update---first across clients and then across tasks---before computing the classifier in closed form. The result is a one-shot analytic update per task that retains the dominant second-order information without repeated client optimization or full Gram matrix communication. A prototype-based extension further incorporates confident unlabeled samples when labels are scarce. 
Our contributions are threefold:

\begin{itemize}
    \item We introduce FedRAN, a low-rank analytic FCL framework that directly compresses local Gram matrices and merges their dominant spectral directions through a two-level QR--SVD update, spatially across clients and temporally across tasks. This produces a rank-bounded global spectral state and reduces second-order communication from $O(M^2)$ to $O(Mr)$.
    
    \item We derive deterministic bounds that separate client-side summarization error from server-side merge error and connect the resulting Gram approximation to classifier error and prediction-score stability.
    
    \item FedRAN improves average accuracy by up to 4.8 percentage points over the strongest baseline, while using 30.6--121.8$\times$ less per-client communication and running 190.3$\times$ faster than optimization-based baselines.
\end{itemize}



\section{Related Works}
\label{sec:related}

\subsection{Federated Continual Learning}

Existing FCL methods mainly address forgetting through optimization-based approaches such as replaying past or synthetic data, exemplar-free distillation, and constraining model updates to preserve prior knowledge~\cite{FedWeiT,dong2023no,qi2023better,target,bakman2023federated,FedTA,fedgtea2025,fedEWC,icarl,fedLWF}. Other approaches reduce the trainable state by adapting pretrained models with prompts, adapters, or other parameter-efficient modules~\cite{FedCPrompt,Pilora,c2prompt2025}. Recent resource-aware and on-device studies further emphasize the need to maintain FCL accuracy under constrained client computation, communication, and label availability~\cite{sacfl2025,fedinc2023,li2026resource}. Motivated by this line of work, FedRAN replaces iterative client training with a forward-only upload of feature statistics per task. Its frozen backbone also avoids client-induced feature drift, while continual adaptation is handled only through the analytic classifier.

\subsection{Analytic Learning and Low-Rank Aggregation}

Analytic continual learning updates classifiers in closed form over fixed features, including incremental and dual-stream formulations~\cite{acil2022,foal2024,dsal2024,gacl2024,anacp2025}. High-dimensional random features can improve feature separability, while low-rank formulations and truncated SVD improve stability in centralized continual learning~\cite{rahimi2007random,rahimi2008kitchen,mcdonnell2023ranpac,randumb2024,loranpac2025}. Federated analytic methods extend these updates across clients by aggregating feature statistics~\cite{fed3r2024,afcl2025,stsa}. Communication-efficient variants avoid communicating full second-order statistics by using class-level summaries or covariance estimates~\cite{fedncm2023,fedcof2025}. Federated PCA, streaming SVD, and matrix sketching provide complementary tools for compact subspace aggregation~\cite{FedPCA,vrehuuvrek2011subspace,eftekhari2019moses,iwen2016distsvd,ghashami2016fd}. FedRAN develops a two-level spatial--temporal QR--SVD merge of low-rank random-feature Gram summaries for FCL. Rather than transmitting the full Gram matrix or reconstructing it from first-order class statistics, FedRAN directly compresses the second-order statistic itself and retains its dominant spectral structure under a rank budget.


\section{Method}
\label{sec:method}

We now describe FedRAN. We first define the FCL setting and centralized analytic reference in Subsection~\ref{subsec:prelim_fcl}, then present the client-side low-rank summarization and spatial--temporal spectral aggregation in Subsections~\ref{subsec:client_summarization} and~\ref{subsec:server_qr_svd}. Subsection~\ref{subsec:analytic_update} gives the analytic classifier and its approximation guarantees, and Subsection~\ref{subsec:pseudo_labeling} introduces the label-scarce extension. Figure~\ref{fig:fedran_methodology} summarizes the complete pipeline.

\begin{figure*}[t]
    \centering
    \includegraphics[width=0.9\linewidth]{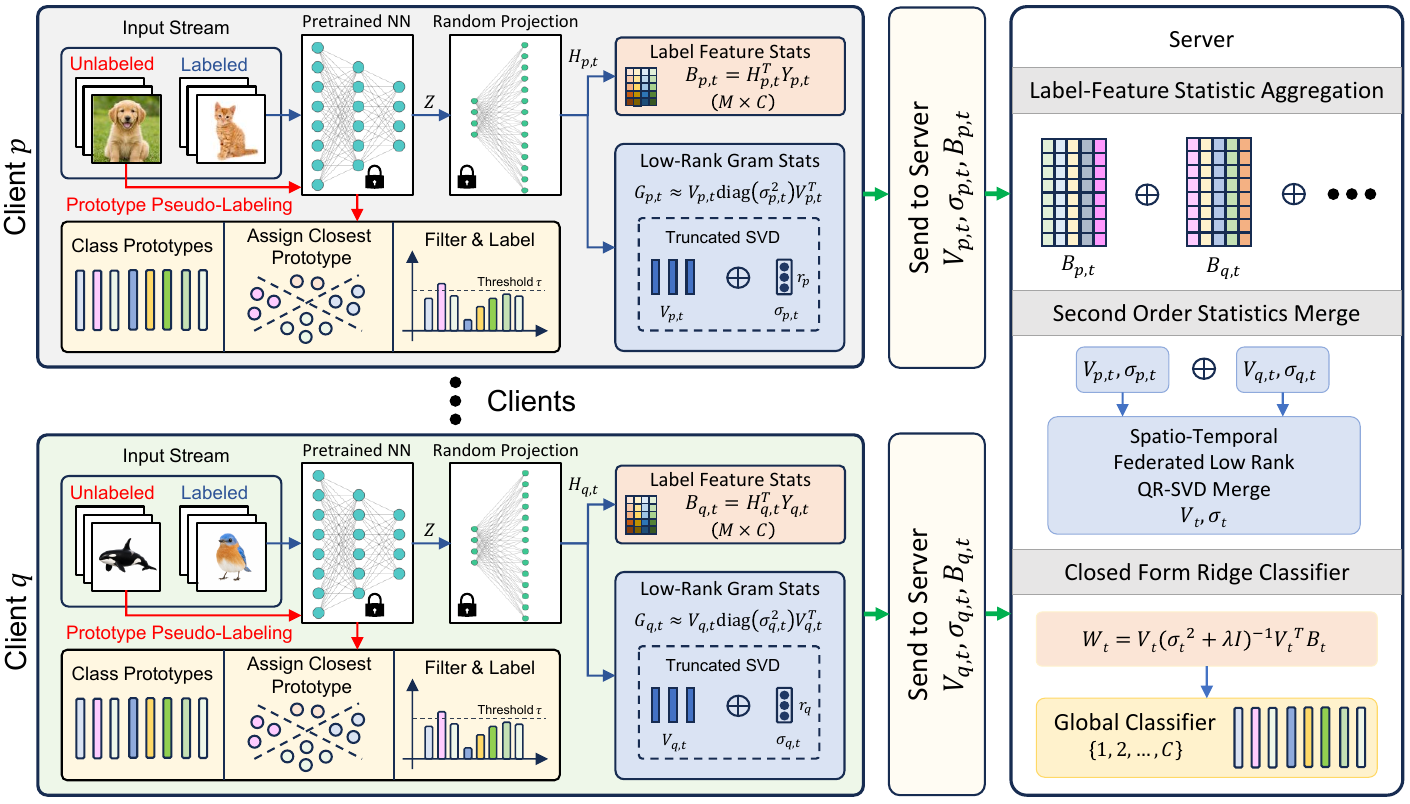}
    \caption{
    Overview of FedRAN.
    Each client extracts features using a frozen pretrained backbone, optionally assigns pseudo-labels to confident unlabeled samples, applies a shared Gaussian random projection, and uploads an exact label-feature statistic together with a low-rank Gram summary.
    The server sums the label-feature statistics, merges the spectral summaries spatially across clients and temporally across tasks using QR--SVD, and solves a ridge classifier in the retained subspace.
    }
    \label{fig:fedran_methodology}
\end{figure*}

\subsection{FCL Setting and Analytic Reference}
\label{subsec:prelim_fcl}

\noindent\textbf{Federated class-incremental stream.}
We consider $K$ clients learning over a sequence of tasks $\{\mathcal{T}_t\}_{t=1}^{T}$. At task $t$, client $k$ receives a private local stream
$\mathcal{D}_{k,t}=\mathcal{D}^{\ell}_{k,t}\cup\mathcal{D}^{u}_{k,t}$, where
$\mathcal{D}^{\ell}_{k,t}=\{(x_{k,t,i},y_{k,t,i})\}_{i=1}^{n^{\ell}_{k,t}}$
is the labeled subset and
$\mathcal{D}^{u}_{k,t}=\{x_{k,t,i}\}_{i=1}^{n^{u}_{k,t}}$
is the unlabeled subset; in the fully supervised setting, $n^{u}_{k,t}=0$.
Raw samples remain local and are never centralized.
Each task introduces a new set of classes $\mathcal{C}_{t}^{\mathrm{new}}$. Let
$\mathcal{C}_{1:t}=\bigcup_{\tau=1}^{t}\mathcal{C}_{\tau}^{\mathrm{new}}$
denote all classes observed up to task $t$, with
$C_t=|\mathcal{C}_{1:t}|$.
After completing task $t$, the global predictor must classify samples from any class in $\mathcal{C}_{1:t}$ without receiving the task identity at inference time.
Let $X_{k,t}=\{x_{k,t,i}\}_{i=1}^{\tilde n_{k,t}}$ denote the samples used in the analytic update, and let
$Y_{k,t}\in\mathbb{R}^{\tilde n_{k,t}\times C_t}$
denote their one-hot label matrix.
In the fully supervised setting, $\tilde n_{k,t}=n^{\ell}_{k,t}$; when pseudo-labeling is enabled, it also includes accepted pseudo-labeled samples.
When new classes arrive, earlier label-dependent matrices are padded with zero columns to match $C_t$.
Clients may differ in sample count, class coverage, and label availability.

\noindent\textbf{Fixed random-feature representation.}
FedRAN uses a pretrained backbone $f(\cdot)$ and keeps it frozen throughout the task stream.
For the samples used in the update, client $k$ computes
\begin{equation}
\begin{aligned}
    Z_{k,t}
    &=f(X_{k,t})
    \in\mathbb{R}^{\tilde n_{k,t}\times d},\\
    H_{k,t}
    &=\mathrm{ReLU}(Z_{k,t}P)
    \in\mathbb{R}^{\tilde n_{k,t}\times M},
    \qquad
    P\in\mathbb{R}^{d\times M},
\end{aligned}
\label{eq:random_features}
\end{equation}
where $d$ is the backbone feature dimension and $M$ is the random-feature dimension.
The projection $P$ is sampled once as a Gaussian random matrix and shared across all clients and tasks, so their statistics are computed in the same random-feature space.
The projection introduces no trainable client-side parameters, while the ReLU nonlinearity creates nonlinear random-feature interactions that improve linear separability~\cite{rahimi2007random,mcdonnell2023ranpac}.
Because both $f(\cdot)$ and $P$ remain fixed, clients do not backpropagate through the backbone or induce representation drift during federated learning~\cite{venkatesha2022addressing}.

\noindent\textbf{Centralized analytic reference.}
Let $H_{1:t}$ and $Y_{1:t}$ denote the row-wise concatenation of the random features and labels from all clients and tasks up to $t$, with earlier label matrices padded to $C_t$ columns.
If these random features were centralized, the full analytic classifier
$W^{\star}_{1:t}\in\mathbb{R}^{M\times C_t}$
would solve
\begin{equation}
\begin{aligned}
W^{\star}_{1:t}
&=
\arg\min_{W}
\|H_{1:t}W-Y_{1:t}\|_F^2
+\lambda\|W\|_F^2\\
&=
\left(G^{\star}_{1:t}+\lambda I_M\right)^{-1}B_{1:t},\\
G^{\star}_{1:t}
&=H_{1:t}^{\top}H_{1:t},
\qquad
B_{1:t}=H_{1:t}^{\top}Y_{1:t},
\end{aligned}
\label{eq:full_ridge_solution}
\end{equation}
where $\lambda>0$ is the ridge parameter.
Here, $G^{\star}_{1:t}$ is the exact accumulated Gram statistic and $B_{1:t}$ is the corresponding label-feature statistic.
FedRAN retains $B_{1:t}$ exactly but replaces the full $M\times M$ Gram statistic with a mergeable low-rank spectral summary.

\subsection{Client-Side Low-Rank Summarization}
\label{subsec:client_summarization}

\noindent\textbf{Exact local statistics.}
The centralized ridge solution in Eq.~\eqref{eq:full_ridge_solution} depends on a Gram statistic and a label-feature statistic. Client $k$ computes its local contributions as
\begin{equation}
\begin{aligned}
    G_{k,t}&=H_{k,t}^{\top}H_{k,t}\in\mathbb{R}^{M\times M},\\
    B_{k,t}&=H_{k,t}^{\top}Y_{k,t}\in\mathbb{R}^{M\times C_t}.
\end{aligned}
\label{eq:local_stats}
\end{equation}
The $c$-th column of $B_{k,t}$ is the sum of the random features assigned to class $c$ at client $k$ and task $t$. Thus, $B_{k,t}$ collects class-wise random-feature sums.

\noindent\textbf{Exact spatial--temporal additivity.}
Since all clients and tasks use the same fixed random-feature map, their local statistics lie in a common feature space and add exactly:
\begin{equation}
\begin{aligned}
    G^{\star}_{1:t}
    &=H_{1:t}^{\top}H_{1:t}
    =\sum_{\tau=1}^{t}\sum_{k=1}^{K}G_{k,\tau},\\
    B_{1:t}
    &=H_{1:t}^{\top}Y_{1:t}
    =\sum_{\tau=1}^{t}\sum_{k=1}^{K}B_{k,\tau}.
\end{aligned}
\label{eq:additive_stats}
\end{equation}
Thus, transmitting the full local statistics would recover the same ridge solution as centralizing all random features and labels observed up to task $t$.

\noindent\textbf{Why low-rank summaries are needed.}
Although $G_{k,t}$ is additive, transmitting it requires $M^2$ floating-point values per client per task. For $M=10{,}000$, this corresponds to $10^8$ values, or roughly $400$ MB in fp32. FedRAN therefore keeps $B_{k,t}$ exact but compresses $G_{k,t}$ through the dominant singular directions of $H_{k,t}$.

\noindent\textbf{Truncated SVD summary.}
Each client computes a rank-$r_{k,t}$ truncated SVD of its random-feature matrix:
\begin{equation}
\begin{aligned}
    H_{k,t}
    &\approx
    U_{k,t}\mathrm{diag}(\sigma_{k,t})V_{k,t}^{\top},\\
    \widetilde G_{k,t}
    &=
    V_{k,t}\mathrm{diag}(\sigma_{k,t}^{2})V_{k,t}^{\top}
    \approx G_{k,t},
\end{aligned}
\label{eq:client_svd}
\end{equation}
where
$V_{k,t}\in\mathbb{R}^{M\times r_{k,t}}$,
$\sigma_{k,t}\in\mathbb{R}^{r_{k,t}}$, and
$r_{k,t}=\min\{r,\tilde n_{k,t},M\}$ for target rank $r$.
The client uploads $(V_{k,t},\sigma_{k,t},B_{k,t})$ to the server, requiring
$Mr_{k,t}+r_{k,t}+MC_t$ values instead of $M^2+MC_t$.
For a per-message budget $\mathcal{B}$ bytes and $b$ bytes per scalar, a feasible rank satisfies
\begin{equation}
    r_{k,t}
    \le
    \left\lfloor
    \frac{\mathcal{B}/b-MC_t}{M+1}
    \right\rfloor.
\label{eq:budget_rank}
\end{equation}

\subsection{Spatial--Temporal Spectral Aggregation}
\label{subsec:server_qr_svd}

\noindent\textbf{Two-level aggregation.}
The server maintains the exact accumulated label-feature statistic $B_{1:t}$ and a rank-bounded spectral summary $(V_{1:t},\sigma_{1:t})$ of the accumulated Gram statistic. At task $t$, it sums the client statistics as $B_t=\sum_{k=1}^{K}B_{k,t}$ and updates $B_{1:t}=B_{1:t-1}+B_t$. It then merges the client spectral summaries in two stages: spatially across clients in the current task to obtain $(V_t,\sigma_t)$, and temporally with the previous global summary to obtain $(V_{1:t},\sigma_{1:t})$. The complete client--server procedure and helper routines are provided in the supplementary material.

\noindent\textbf{QR--SVD merge.}
Consider two spectral summaries $(V_a,\sigma_a)$ and $(V_b,\sigma_b)$, where
$V_a\in\mathbb{R}^{M\times r_a}$ and
$V_b\in\mathbb{R}^{M\times r_b}$ have orthonormal columns, and
$\sigma_a$ and $\sigma_b$ contain the corresponding nonnegative singular values. FedRAN forms
\begin{equation}
    A=
    \left[
        V_a\mathrm{diag}(\sigma_a),\;
        V_b\mathrm{diag}(\sigma_b)
    \right]
    \in\mathbb{R}^{M\times(r_a+r_b)}.
\label{eq:merge_A}
\end{equation}
The identity
$AA^{\top}
=
V_a\mathrm{diag}(\sigma_a^2)V_a^{\top}
+
V_b\mathrm{diag}(\sigma_b^2)V_b^{\top}$
shows that $A$ exactly represents the sum of the two input Gram sketches. FedRAN computes a thin QR factorization $A=QR$, followed by
$R=U_R\mathrm{diag}(\bar\sigma)W_R^{\top}$, keeps up to the top $r$ singular components, and returns
$V=QU_R^{(:,1:r)}$ and $\sigma=\bar\sigma_{1:r}$.
This operation avoids materializing an $M\times M$ matrix. When both inputs have rank $r$, it operates on $A\in\mathbb{R}^{M\times 2r}$ and costs $O(Mr^2+r^3)$ per merge. The same operation is used for spatial aggregation across clients and temporal aggregation across tasks. Related QR--SVD merging ideas have been studied in federated PCA and streaming subspace tracking~\cite{FedPCA,vrehuuvrek2011subspace,eftekhari2019moses}; FedRAN uses this operation to maintain a spatial--temporal random-feature state for analytic FCL.

\begin{fedranprop}[QR--SVD merge of Gram sketches]
\label{prop:qr_svd_covariance}
Let $A$ be defined as in Eq.~\eqref{eq:merge_A}, and let
$S=AA^\top$ have eigenvalues
$\lambda_1(S)\ge\lambda_2(S)\ge\cdots\ge0$.
If all singular components are retained, the merged factors $(V,\sigma)$ satisfy
\begin{equation}
\begin{aligned}
    V\mathrm{diag}(\sigma^2)V^{\top}
    &=AA^{\top}\\
    &=
    V_a\mathrm{diag}(\sigma_a^2)V_a^{\top}
    +
    V_b\mathrm{diag}(\sigma_b^2)V_b^{\top}.
\end{aligned}
\label{eq:merge_exact_covariance}
\end{equation}
If only the top $r$ components are retained, then
$V\mathrm{diag}(\sigma^2)V^{\top}=\mathcal{T}_r(S)$,
the best rank-$r$ approximation to $S$ in both spectral and Frobenius norms, with residual
\begin{equation}
\begin{aligned}
    \|S-\mathcal{T}_r(S)\|_2
    &=\lambda_{r+1}(S),\\
    \|S-\mathcal{T}_r(S)\|_F
    &=
    \left(\sum_{i>r}\lambda_i(S)^2\right)^{1/2}.
\end{aligned}
\label{eq:merge_residual_norms}
\end{equation}
\end{fedranprop}

The result follows from the QR--SVD factorization and the Eckart--Young--Mirsky theorem; the proof is provided in Appendix~\ref{app:proof_qr_svd} of the supplementary material.

\noindent\textbf{Approximation to the accumulated Gram statistic.}
Let
$\widetilde G_{1:t}
=
V_{1:t}\mathrm{diag}(\sigma_{1:t}^{2})V_{1:t}^{\top}$
denote the FedRAN Gram sketch after all local summaries and spatial--temporal merges up to task $t$. The following theorem characterizes its error through the directions discarded locally by clients and during server-side merging.

\begin{fedranthm}[FedRAN Gram approximation error]
\label{thm:gram_error}
Let
$G^{\star}_{1:t}
=
\sum_{\tau=1}^{t}\sum_{k=1}^{K}G_{k,\tau}$
be the exact accumulated Gram statistic.
For each local random-feature matrix $H_{k,\tau}$, let
$s_{k,\tau,1}\ge s_{k,\tau,2}\ge\cdots$
denote its full singular-value spectrum.
For each QR--SVD merge $j$, let
$S_j=A_jA_j^\top$
denote the matrix before rank truncation, with eigenvalues
$\lambda_{j,1}\ge\lambda_{j,2}\ge\cdots\ge0$.
Then
\begin{equation}
\begin{aligned}
    \|G^{\star}_{1:t}-\widetilde G_{1:t}\|_2
    &\le
    \sum_{\tau=1}^{t}\sum_{k=1}^{K}
    s_{k,\tau,r_{k,\tau}+1}^{2}
    +
    \sum_{j\in\mathcal{M}_t}
    \lambda_{j,r+1}\\
    &\triangleq \varepsilon_G(t),
\end{aligned}
\label{eq:gram_bound}
\end{equation}
where $\mathcal{M}_t$ is the set of spatial and temporal QR--SVD merges performed up to task $t$. If an omitted singular value or eigenvalue does not exist, its contribution zero.
\end{fedranthm}

Theorem~\ref{thm:gram_error} separates the approximation error into client-side truncation and server-side merge truncation. The retained rank controls both sources: larger ranks preserve more spectral directions but increase communication and server computation. The proof and a relative-tail form are provided in the supplementary material.

\subsection{Analytic Classifier \& Approximation Guarantees}
\label{subsec:analytic_update}

\noindent\textbf{Subspace-constrained ridge classifier.}
After aggregation, the server has the exact accumulated label-feature statistic $B_{1:t}$ and the rank-bounded spectral summary $(V_{1:t},\sigma_{1:t})$. Let $\bar r_t=\mathrm{cols}(V_{1:t})\le r$ denote the current rank of the global spectral state. A direct full-space ridge solve would require materializing and inverting the $M\times M$ Gram matrix $G^{\star}_{1:t}$. FedRAN instead restricts the classifier to the retained spectral subspace:
$W_{1:t}=V_{1:t}\widetilde W_{1:t}$, where
$\widetilde W_{1:t}\in\mathbb{R}^{\bar r_t\times C_t}$.
The corresponding objective is
\[
\widetilde W_{1:t}^{\star}
=
\arg\min_{\widetilde W\in\mathbb{R}^{\bar r_t\times C_t}}
\|H_{1:t}V_{1:t}\widetilde W-Y_{1:t}\|_F^2
+\lambda\|\widetilde W\|_F^2.
\]
If the exact Gram matrix were available, its solution would be
$\widetilde W_{1:t}^{\star}
=
(V_{1:t}^{\top}G^{\star}_{1:t}V_{1:t}+\lambda I_{\bar r_t})^{-1}
V_{1:t}^{\top}B_{1:t}$.
FedRAN instead uses the merged Gram sketch
$\widetilde G_{1:t}
=
V_{1:t}\mathrm{diag}(\sigma_{1:t}^{2})V_{1:t}^{\top}$,
for which
$V_{1:t}^{\top}\widetilde G_{1:t}V_{1:t}
=
\mathrm{diag}(\sigma_{1:t}^{2})$.
The resulting classifier is
\begin{equation}
\begin{aligned}
    \widetilde W_{1:t}
    &=
    \left(\mathrm{diag}(\sigma_{1:t}^{2})+\lambda I_{\bar r_t}\right)^{-1}
    V_{1:t}^{\top}B_{1:t},\\
    W_{1:t}
    &=V_{1:t}\widetilde W_{1:t}.
\end{aligned}
\label{eq:fedran_classifier}
\end{equation}
Thus, FedRAN solves ridge regression in the retained subspace and requires only a $\bar r_t$-dimensional diagonal solve rather than a dense $M$-dimensional solve.

\noindent\textbf{Approximation to centralized ridge.}
The following theorem relates the FedRAN classifier to the full ridge solution in Eq.~\eqref{eq:full_ridge_solution}. Its two terms capture the error from approximating the Gram statistic and from restricting the classifier to the retained subspace.

\begin{fedranthm}[Approximation to full ridge]
\label{thm:weight_error}
Let
$W^{\star}_{1:t}
=
(G^{\star}_{1:t}+\lambda I_M)^{-1}B_{1:t}$
be the full ridge solution and let $W_{1:t}$ be the FedRAN classifier in Eq.~\eqref{eq:fedran_classifier}. Let
$\widetilde G_{1:t}
=
V_{1:t}\mathrm{diag}(\sigma_{1:t}^{2})V_{1:t}^{\top}$,
and let $\varepsilon_G(t)$ satisfy Theorem~\ref{thm:gram_error}. Then
\begin{equation}
\begin{aligned}
    \|W^{\star}_{1:t}-W_{1:t}\|_F
    &\le
    \frac{\varepsilon_G(t)}{\lambda^2}\|B_{1:t}\|_F\\
    &\quad+
    \frac{1}{\lambda}
    \|(I_M-V_{1:t}V_{1:t}^{\top})B_{1:t}\|_F.
\end{aligned}
\label{eq:weight_error_bound}
\end{equation}
\end{fedranthm}

The first term propagates the spectral error of the Gram sketch to the classifier. The second measures the label-feature signal outside the retained subspace. FedRAN therefore approaches the full ridge classifier when the Gram sketch is accurate and $B_{1:t}$ lies mostly in the retained spectral directions. The proof is provided in the supplementary material.

\begin{fedrancor}[Score and prediction stability]
\label{cor:score_margin}
For a test random feature $h\in\mathbb{R}^{M}$, let
$s^{\star}=h^{\top}W^{\star}_{1:t}$ and
$s=h^{\top}W_{1:t}$.
If $\varepsilon_W(t)$ denotes the right-hand side of Eq.~\eqref{eq:weight_error_bound}, then
$\|s^{\star}-s\|_2\le \|h\|_2\varepsilon_W(t)$.
Let $y^{\star}=\arg\max_c s^{\star}_c$.
If
\[
s^{\star}_{y^{\star}}
-
\max_{c\ne y^{\star}}s^{\star}_c
>
2\|h\|_2\varepsilon_W(t),
\]
then FedRAN predicts the same class $y^{\star}$ as the full ridge classifier.
\end{fedrancor}

\noindent\textbf{Inference and cost.}
For a test sample $x$, FedRAN computes
$h(x)=\mathrm{ReLU}(f(x)P)$,
$s(x)=h(x)W_{1:t}$, and
$\hat y=\arg\max_{c\in\{1,\ldots,C_t\}}[s(x)]_c$.
After computing the fixed random feature $h(x)$, inference requires a single linear score computation.
FedRAN replaces iterative model exchange with one statistic upload per client and task, and replaces the full Gram solve with rank-$r$ QR--SVD merging and a diagonal subspace solve. Detailed communication and computation costs are provided in the supplementary material.

\subsection{Label-Scarce Extension}
\label{subsec:pseudo_labeling}

\noindent\textbf{Prototype-based pseudo-labeling.}
The label-feature statistic
$B_{k,t}=H_{k,t}^{\top}Y_{k,t}$
requires labels for the samples contributing to the analytic update. In resource-constrained FCL, however, labels may be sparse, delayed, or unavailable for much of the incoming stream~\cite{li2026resource}. Using only labeled samples can therefore under-utilize the available local data. FedRAN addresses this setting by assigning pseudo-labels to high-confidence unlabeled samples before constructing the local statistics.

\noindent\textbf{Local class prototypes.}
Let $Z^{\ell}_{k,t}$ and $y^{\ell}_{k,t}$ denote the labeled backbone features and labels at client $k$, and let $Z^{u}_{k,t}$ denote its unlabeled backbone features. For each class $c$ represented in the labeled subset, the client computes
\begin{equation}
\begin{aligned}
    p_{k,t,c}
    &=
    \frac{1}{|\mathcal{I}_{k,t,c}|}
    \sum_{i\in\mathcal{I}_{k,t,c}} Z^{\ell}_{k,t}[i],\\
    \mathcal{I}_{k,t,c}
    &=
    \{i:y^{\ell}_{k,t}[i]=c\},
\end{aligned}
\label{eq:ssl_proto}
\end{equation}
and normalizes the prototype as
$p_{k,t,c}\leftarrow p_{k,t,c}/\|p_{k,t,c}\|_2$.

\noindent\textbf{Confidence-filtered pseudo-labels.}
For each unlabeled feature $u\in Z^{u}_{k,t}$, FedRAN computes
$\bar u=u/\|u\|_2$
and assigns the candidate pseudo-label
$c^{\star}(u)=\arg\max_c\bar u^{\top}p_{k,t,c}$.
The pseudo-label is accepted only if
$\bar u^{\top}p_{k,t,c^{\star}(u)}\ge\tau$,
where $\tau$ is the confidence threshold.
Accepted pseudo-labeled features are combined with the labeled features to form
$(\widetilde Z_{k,t},\widetilde y_{k,t})$.
The corresponding random features and one-hot labels are then used in
Eqs.~\eqref{eq:random_features} and~\eqref{eq:local_stats}.
The guarantees above condition on the label matrix used by FedRAN; when pseudo-labeling is enabled, this is the augmented label matrix.
\section{Experiments}
\label{evaluation}

\begin{figure*}
    \centering
    \includegraphics[width=0.98\textwidth]{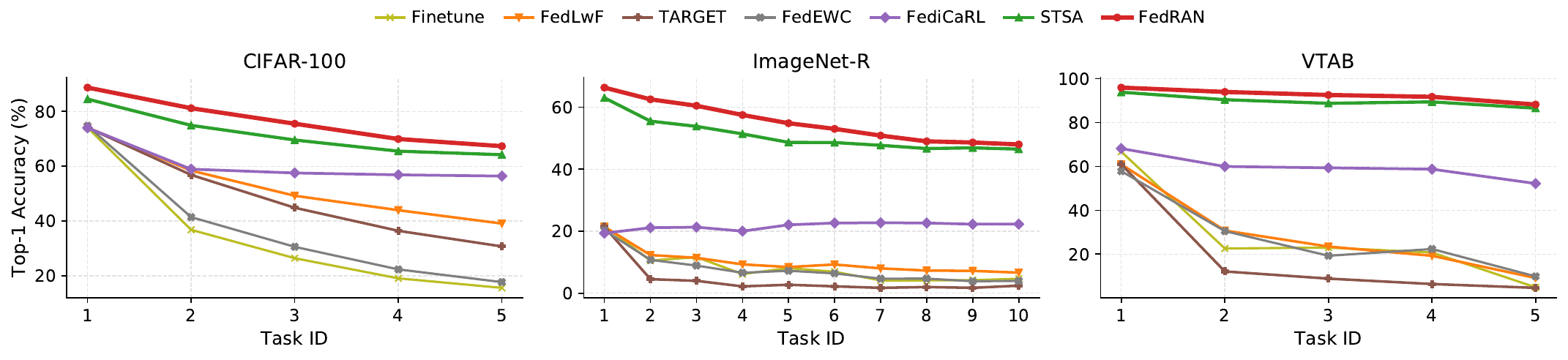}
    \caption{Task-wise accuracy comparison across CIFAR-100, ImageNet-R, and VTAB with $\beta=0.1$. 
    }
    \label{fig:task_accuracy}
\end{figure*}

\subsection{Experimental Setup}

We evaluate FedRAN in the federated class-incremental learning (FCIL) setting.
For the federated setup, we use $K=5$ clients. 
For non-IID client data, the training samples of each task are partitioned across clients using a Dirichlet distribution~\cite{hsu2019measuring} with parameter $\beta=[0.1,  0.5, 1]$. A smaller $\beta$ produces more skewed class distributions across clients, resulting in stronger non-IID partitions. We evaluate two backbone settings, ResNet-18~\cite{resnet} and ViT-B/16~\cite{vitb16}, pretrained on ImageNet. 
The ResNet setting uses projection dimension $M=8192$ with rank $r=2048$ on CIFAR-100 and ImageNet-R, and rank $r=512$ on VTAB. The ViT setting uses $M=2048$ and $r=512$. For prototype-based pseudo-labeling, the confidence threshold is set to $\tau=0.5$. 

\subsection{Datasets and Baselines}

We evaluate FedRAN on three widely used continual learning image classification benchmarks. CIFAR-100~\cite{cifar100} is a standard image classification benchmark. ImageNet-R~\cite{imagenetR} evaluates robustness under out-of-distribution visual renditions. VTAB~\cite{vtab} spans diverse visual domains, including natural, medical, and remote sensing images. For CIFAR-100, we use 5 tasks in the ResNet setting and 10 tasks in the ViT setting, consistent with prior works~\cite{target, stsa, FedCPrompt}. 
We compare FedRAN against representative baselines for FCL: Finetune~\cite{FedAvg}, FedEWC~\cite{fedEWC}, FedLwF~\cite{fedLWF}, FediCaRL~\cite{icarl}, TARGET~\cite{target}, and STSA-E~\cite{stsa}.
For pretrained ViT-based comparisons, we compare against DualPrompt~\cite{dualIP}, CodaPrompt~\cite{CodaP}, Fed-CPrompt~\cite{FedCPrompt},  PiLoRA~\cite{Pilora}, and STSA-E~\cite{stsa}.

\subsection{Metrics}
We evaluate FedRAN using two sets of metrics: accuracy and efficiency. \textbf{Accuracy metrics:} (i) final accuracy $A_T$, measured after the last task over all observed classes; and (ii) average accuracy $A_{\mathrm{avg}}$, computed by averaging task accuracy across the task stream. At each task, accuracy is evaluated over all classes observed up to that task. \textbf{Efficiency metrics:} (i) communication cost, measured in MB as the maximum amount of data transmitted by a client in a single task; and (ii) runtime, measured as wall-clock training time in seconds, including both client-side and server-side computation for each task.

\subsection{Results}
\subsubsection{\textbf{Overall accuracy}}
Table~\ref{tab:main_accuracy_aavg_at} compares FedRAN with ResNet-based FCL baselines under different Dirichlet settings, using the same number of clients, task splits, and data partitioning across clients. 
FedRAN achieves the best $A_{\mathrm{avg}}$ and $A_T$ across all configurations.
The task accuracy curves in Fig.~\ref{fig:task_accuracy} show that optimization-based baselines, such as Finetune, exhibit larger accuracy drops as the number of tasks increases, indicating stronger forgetting under sequential client updates. 
In contrast, STSA-E and FedRAN maintain higher accuracy across the task stream. This suggests that freezing the backbone helps stabilize the client drift and reduces catastrophic forgetting. FedRAN further improves over STSA-E by using an SVD-based low-rank summary of the projected feature matrix, which preserves the dominant feature directions needed for the analytic classifier update. 

\begin{table}
\centering
\caption{Performance comparison under different Dirichlet settings. Each dataset reports average accuracy ($A_{\mathrm{avg}}$) and final accuracy ($A_T$). Improvement denotes the absolute percentage-point gain of FedRAN over the strongest baseline for each metric. Best and second-best results are shown in \textbf{bold} and \underline{underlined}.}
\label{tab:main_accuracy_aavg_at}
\setlength{\tabcolsep}{5.0pt}
\resizebox{\columnwidth}{!}{%
\begin{tabular}{@{}llcccccc@{}}
\toprule
Method & $\beta$
& \multicolumn{2}{c}{CIFAR-100}
& \multicolumn{2}{c}{ImageNet-R}
& \multicolumn{2}{c}{VTAB} \\
\cmidrule(lr){3-4}
\cmidrule(lr){5-6}
\cmidrule(l){7-8}
& & $A_{\mathrm{avg}}$ & $A_T$
& $A_{\mathrm{avg}}$ & $A_T$
& $A_{\mathrm{avg}}$ & $A_T$ \\
\midrule
Finetune~\cite{FedAvg}
& 0.1 & 34.43 & 15.65 & 8.22 & 4.67 & 27.57 & 4.86 \\
& 0.5 & 36.09 & 16.57 & 9.20 & 4.63 & 35.56 & 19.65 \\
& 1.0 & 37.02 & 17.12 & 9.55 & 4.68 & 36.57 & 18.92 \\
\midrule
FedLwF~\cite{fedLWF}
& 0.1 & 52.97 & 39.11 & 10.13 & 6.64 & 28.70 & 9.12 \\
& 0.5 & 55.71 & 38.24 & 11.74 & 7.92 & 38.31 & 24.12 \\
& 1.0 & 59.79 & 44.35 & 12.91 & 6.99 & 41.38 & 14.43 \\
\midrule
TARGET~\cite{target}
& 0.1 & 48.61 & 30.76 & 4.49 & 2.43 & 18.57 & 4.61 \\
& 0.5 & 51.84 & 34.27 & 5.86 & 2.12 & 20.59 & 4.63 \\
& 1.0 & 53.58 & 35.02 & 7.04 & 2.98 & 25.45 & 5.93 \\
\midrule
FedEWC~\cite{fedEWC}
& 0.1 & 37.39 & 17.75 & 7.74 & 3.97 & 27.95 & 9.83 \\
& 0.5 & 41.46 & 21.42 & 9.58 & 4.62 & 36.06 & 20.93 \\
& 1.0 & 42.81 & 21.70 & 9.69 & 4.38 & 38.10 & 18.45 \\
\midrule
FediCaRL~\cite{icarl}
& 0.1 & 60.76 & 56.42 & 21.65 & 22.29 & 59.68 & 52.19 \\
& 0.5 & 66.05 & 61.10 & 23.84 & 25.49 & 68.14 & 63.97 \\
& 1.0 & 68.76 & 62.51 & 24.07 & 24.29 & 72.95 & 64.07 \\
\midrule
STSA-E~\cite{stsa}
& 0.1 & \underline{71.74} & \underline{64.21} & \underline{50.89} & \underline{46.50} & \underline{89.81} & \underline{86.55} \\
& 0.5 & \underline{73.53} & \underline{66.06} & \underline{51.37} & \underline{47.92} & \underline{90.59} & \underline{87.27} \\
& 1.0 & \underline{74.37} & \underline{66.99} & \underline{51.77} & \underline{47.85} & \underline{90.38} & \underline{87.05} \\
\midrule
FedRAN (Ours)
& 0.1 & \textbf{76.54} & \textbf{67.33} & \textbf{55.13} & \textbf{48.00} & \textbf{92.49} & \textbf{88.22} \\
& 0.5 & \textbf{76.50} & \textbf{67.32} & \textbf{55.16} & \textbf{48.13} & \textbf{92.49} & \textbf{88.22} \\
& 1.0 & \textbf{76.66} & \textbf{67.60} & \textbf{55.30} & \textbf{48.35} & \textbf{92.49} & \textbf{88.22} \\
\midrule
Improvement
& 0.1 & +4.80 & +3.12 & +4.24 & +1.50 & +2.68 & +1.67 \\
& 0.5 & +2.97 & +1.26 & +3.79 & +0.21 & +1.90 & +0.95 \\
& 1.0 & +2.29 & +0.61 & +3.53 & +0.50 & +2.11 & +1.17 \\
\bottomrule
\end{tabular}%
}
\end{table}

\subsubsection{\textbf{Communication overhead}}
Table~\ref{tab:comm_cost} reports the communication cost per client per task. 
TARGET is shown as a representative optimization-based FCL method, since Finetune, FedLwF, FedEWC, and FediCaRL also communicate model updates through iterative training rounds and therefore have comparable communication costs. FedRAN substantially reduces communication costs compared with TARGET, whose costs are $31.90\times$, $30.65\times$, and $121.75\times$ higher on CIFAR-100, ImageNet-R, and VTAB, respectively. Compared with STSA-E, FedRAN also reduces communication by $1.45\times$ on CIFAR-100 and $2.77\times$ on both ImageNet-R and VTAB.
This reduction comes from the client summary design. In our experimental setup, we consider the communication-efficient STSA-E method, which communicates class-wise feature sums and class counts, including the dummy-class statistics used to estimate the Gram matrix. In contrast, FedRAN directly transmits a low-rank spectral summary.

\begin{table}
\centering
\caption{Communication cost per client per task. Costs are reported in MB, with the relative cost normalized to FedRAN shown in parentheses; lower is better. TARGET is shown as a representative of all optimization-based methods such as Finetune, FedLwF, FedEWC, and FediCaRL.}
\label{tab:comm_cost}
\scriptsize
\setlength{\tabcolsep}{3.2pt}
\begin{tabular}{lccc}
\toprule
Method & CIFAR-100 & ImageNet-R & VTAB \\
\midrule
\begin{tabular}[c]{@{}l@{}} TARGET~\cite{target}\end{tabular}
& 4283.82 (31.90$\times$)
& 4306.32 (30.65$\times$)
& 4276.96 (121.75$\times$) \\
STSA-E~\cite{stsa}
& 194.59 (1.45$\times$)
& 389.18 (2.77$\times$)
& 97.29 (2.77$\times$) \\
FedRAN
& \textbf{134.27} (1.00$\times$)
& \textbf{140.52} (1.00$\times$)
& \textbf{35.13} (1.00$\times$) \\
\bottomrule
\end{tabular}
\end{table}

\subsubsection{\textbf{Computation Overhead}}
Table~\ref{tab:runtime} reports the runtime per client per task, including the computation performed by a single client and the corresponding server-side update for that task. Optimization-based baselines such as Finetune, FedLwF, TARGET, FedEWC, and FediCaRL require local backpropagation over multiple rounds and take hundreds of seconds per task. In contrast, FedRAN completes each task in a few seconds by avoiding iterative local training over multiple rounds. Across the gradient-based baselines, FedRAN is on average $190.3\times$ faster.
FedRAN is slightly slower than STSA-E on CIFAR-100 and ImageNet-R due to the additional SVD and QR--SVD spectral aggregation, but it achieves consistently higher accuracy, as shown in Table~\ref{tab:main_accuracy_aavg_at}. 

\begin{table}
\centering
\caption{Runtime per client per task (seconds). Lower is better.}
\label{tab:runtime}
\resizebox{0.88\linewidth}{!}{
\begin{tabular}{lccc}
\toprule
Method & CIFAR-100 & ImageNet-R & VTAB \\
\midrule
Finetune~\cite{FedAvg} & 322.28 & 589.72 & 229.37 \\
FedLwF~\cite{fedLWF} & 316.48 & 559.71 & 194.83 \\
TARGET~\cite{target} & 320.18 & 567.28 & 196.62 \\
FedEWC~\cite{fedEWC} & 409.49 & 698.43 & 266.81 \\
FediCaRL~\cite{icarl} & 346.44 & 646.68 & 225.45 \\
STSA-E~\cite{stsa} & \textbf{1.75} & \textbf{2.13} & \underline{1.01} \\
FedRAN (Ours) & \underline{3.54} & \underline{2.48} & \textbf{0.98} \\
\bottomrule
\end{tabular}
}
\end{table}

\begin{figure}
    \centering
    \includegraphics[width=\columnwidth]{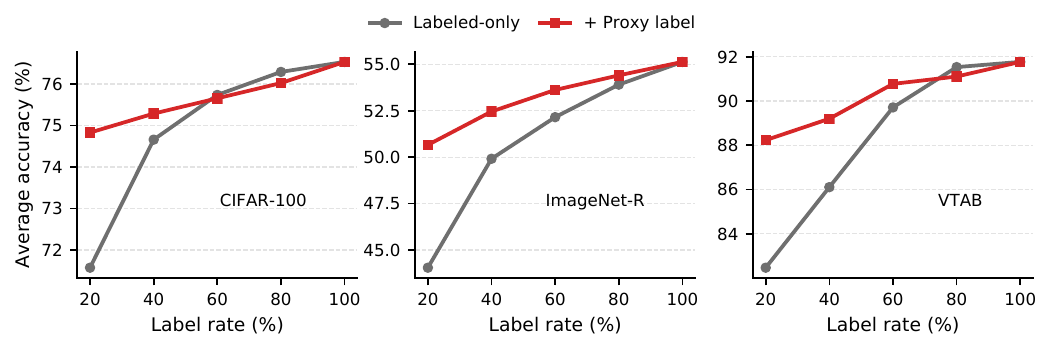}
    \caption{Effect of proxy-label under limited labeled data. FedRAN with proxy labels consistently improves average accuracy at low labeled rates, with the largest gains at 20\% labeled data.}
    \label{fig:label_rate}
\end{figure}

\subsubsection{\textbf{Effect of Proxy-Labeling}}
Figure~\ref{fig:label_rate} evaluates FedRAN under limited labeled data at $\beta=0.1$. 
Proxy-labeling provides the largest gains when only $20\%$ of the data is labeled, improving $A_{\mathrm{avg}}$ by $3.26$ on CIFAR-100, $6.61$ on ImageNet-R, and $5.76$ on VTAB. As the label rate increases, the gain decreases because the labeled-only class-prototype matrix becomes more reliable. These results show that proxy-labeling improves FedRAN under label scarcity by allowing confident unlabeled samples to contribute to feature statistics.

\begin{figure}
    \centering
    \includegraphics[width=\columnwidth]{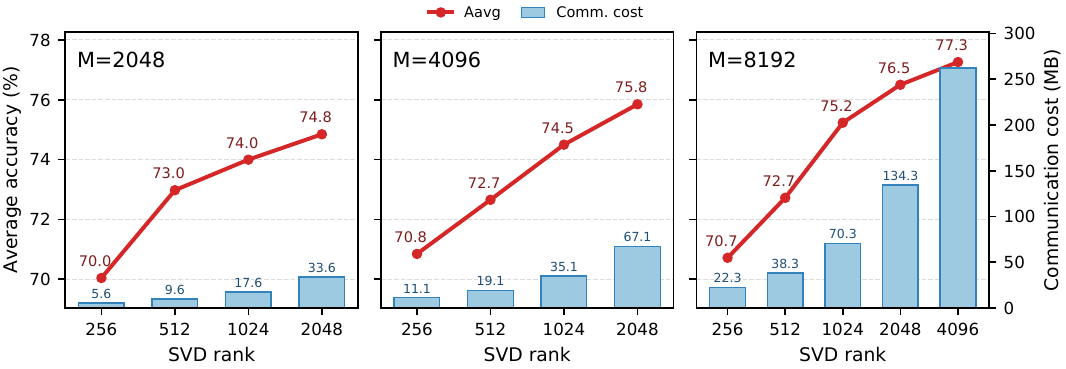}
    \caption{Effect of the SVD rank under different random projection dimensions $M$ on FedRAN for CIFAR-100 with $\beta=0.5$.} 
    \label{fig:rank_ablation}
\end{figure}

\subsubsection{\textbf{Effect of Projection Dimension and Rank}}


Figure~\ref{fig:rank_ablation} shows the effect of SVD rank $r$ and random projection dimension $M$ on CIFAR-100 with $\beta=0.5$. Increasing $r$ improves $A_{\mathrm{avg}}$ as the client's low-rank spectral summary retains more dominant directions of the local Gram matrix. However, the gain saturates at larger ranks.
This suggests that moderate ranks already capture most of the useful spectral information.
Increasing $M$ also improves accuracy by increasing feature separability, but it also further increases communication cost. 
These results show the trade-off between accuracy and communication cost. 
\begin{figure}[t]
    \centering
    \includegraphics[width=0.70\columnwidth]
    {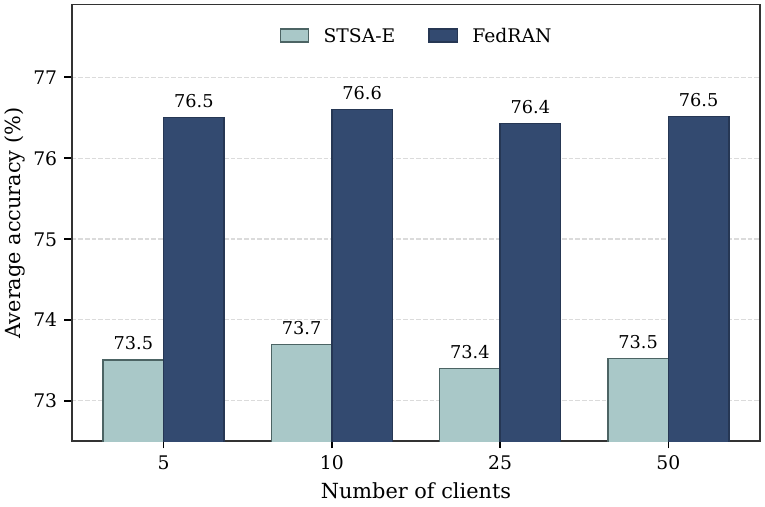}
    \caption{Average accuracy on CIFAR-100 with ResNet-18 at $\beta=0.5$ for different numbers of clients.}
    \label{fig:client_sweep_cifar100}
\end{figure}

\begin{figure}[htbp]
    \centering
    \includegraphics[width=0.70\columnwidth]{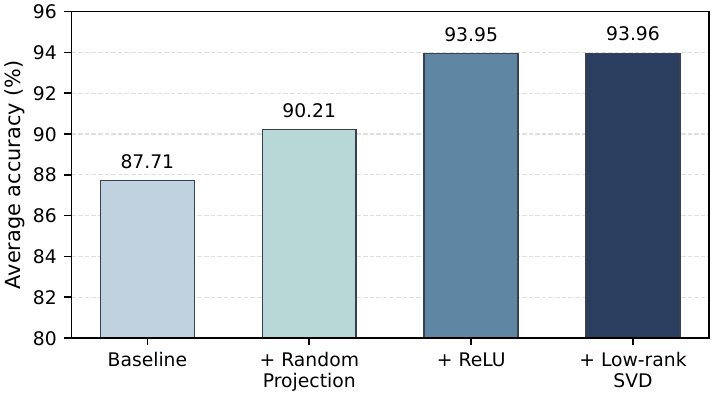}
    \caption{Ablation of FedRAN components on CIFAR-100.
    } 
    \label{fig:model_components_ablation}
\end{figure}

\subsubsection{\textbf{Effect of Number of Clients}}

Figure~\ref{fig:client_sweep_cifar100} shows the effect of the number of clients on FedRAN performance and compares it against the strongest baseline STSA-E~\cite{stsa}. For STSA-E, we vary the number of dummy clients to keep the communication costs for FedRAN and STSA-E close for a fair comparison. FedRAN maintains \(76.4\)--\(76.6\%\) average accuracy across all client counts and consistently outperforms STSA-E. This demonstrates that FedRAN's low-rank spectral aggregation is robust to varying numbers of clients.

\subsubsection{\textbf{Ablation on Model Components}}

To evaluate the architectural necessity of each FedRAN component, we perform an ablation study on CIFAR-100 with the ViT backbone for $\beta=1.0$ as shown in Figure \ref{fig:model_components_ablation}. We establish a baseline without random projection, ReLU, and low-rank SVD, achieving 87.71\% accuracy. Linearly expanding this feature space to $M=1250$ using random projection increases the model's representational capacity, thereby increasing the average accuracy to 90.21\%. Applying the ReLU activation further boosts accuracy to 93.95\%, validating the necessity of the non-linear operation and projection. We finally integrate the low-rank SVD approximation to complete the FedRAN framework, which incurs no accuracy loss. This shows that FedRAN successfully preserves the key second-order directions while reducing the communication cost.

\subsubsection{\textbf{Ablation on ViT as backbone model}}

Table~\ref{tab:vit_accuracy_aavg_at} compares FedRAN with ViT-based FCL baselines using the same number of clients, data partitioning across clients, and tasks. All methods use the same pretrained ViT-B/16 model, with task 0 used for finetuning following the prior ViT-based setting~\cite{stsa}. FedRAN achieves the best $A_{\mathrm{avg}}$ across all $\beta$ values on both datasets. Compared with STSA-E, the strongest baseline, FedRAN improves $A_{\mathrm{avg}}$ by up to $1.03$ percentage points on CIFAR-100 and $2.40$ percentage points on ImageNet-R. 
FedRAN also uses less communication than STSA-E in the ViT setting. On CIFAR-100, FedRAN requires $9.57$ MB per client per task compared with $10.50$ MB for STSA-E. On ImageNet-R, FedRAN requires $11.13$ MB compared with $21.00$ MB for STSA-E. Thus, FedRAN provides a stronger trade-off between accuracy and communication cost.

\begin{table}[ht]
\centering
\caption{ViT-B/16 results under different Dirichlet settings.} 
\label{tab:vit_accuracy_aavg_at}
\small
\setlength{\tabcolsep}{5.0pt}
\begin{tabular}{@{}llcccc@{}}
\toprule
Method & $\beta$
& \multicolumn{2}{c}{CIFAR-100}
& \multicolumn{2}{c}{ImageNet-R} \\
\cmidrule(lr){3-4}
\cmidrule(l){5-6}
& & $A_{\mathrm{avg}}$ & $A_T$
& $A_{\mathrm{avg}}$ & $A_T$ \\
\midrule
Fed-DualP~\cite{dualIP}
& 0.1 & 69.11 & 56.31 & 35.58 & 27.92 \\
& 0.5 & 78.12 & 65.77 & 53.37 & 45.25 \\
& 1.0 & 83.02 & 75.07 & 56.61 & 48.58 \\
\midrule
Fed-CODAP~\cite{CodaP}
& 0.1 & 67.75 & 55.54 & 23.22 & 15.15 \\
& 0.5 & 74.80 & 64.02 & 48.49 & 43.00 \\
& 1.0 & 82.20 & 74.12 & 55.10 & 51.13 \\
\midrule
Fed-CPrompt~\cite{FedCPrompt}
& 0.1 & 69.67 & 61.68 & 23.67 & 16.87 \\
& 0.5 & 75.07 & 65.87 & 45.42 & 39.12 \\
& 1.0 & 81.43 & 73.86 & 51.09 & 47.28 \\
\midrule
PiLoRA~\cite{Pilora}
& 0.1 & 82.85 & 74.09 & 51.85 & 47.85 \\
& 0.5 & 83.46 & 76.81 & 53.43 & 50.18 \\
& 1.0 & 83.59 & 77.17 & 55.47 & 51.92 \\
\midrule
STSA-E~\cite{stsa}
& 0.1 & \underline{92.68} & \underline{88.73} & \underline{65.05} & \underline{58.34} \\
& 0.5 & \underline{93.30} & \underline{89.74} & \underline{65.38} & \underline{58.65} \\
& 1.0 & \underline{93.34} & \underline{89.68} & \underline{68.60} & \underline{61.29} \\
\midrule
FedRAN (Ours)
& 0.1 & \textbf{93.71} & \textbf{89.61} & \textbf{67.45} &  \textbf{59.22} \\
& 0.5 & \textbf{93.91} & \textbf{89.90} & \textbf{67.76} &  \textbf{59.57} \\
& 1.0 & \textbf{93.96} & \textbf{89.82} & \textbf{69.80} &  \textbf{63.54} \\
\bottomrule
\end{tabular}
\end{table}
\section{Conclusion}

We presented FedRAN, a low-rank analytic FCL framework that avoids iterative client optimization.
Across CIFAR-100, ImageNet-R, and VTAB, FedRAN improves average accuracy by up to $4.8$ points while using $30.6$--$121.8\times$ less per-client communication and training $190.3\times$ faster.

\emph{AI Disclosure.} AI tools, including ChatGPT, were used as general assistants to help with refining the manuscript. The authors reviewed and verified the final content.

\section{Acknowledgment}
This work is supported by the DEVCOM Army Research Office (ARO) under grant W911NF261A164 and the National Science Foundation (NSF) under grant OAC-2505107. Approved for public release: distribution is unlimited.

{
    \small
    \bibliographystyle{ieeenat_fullname}
    \bibliography{main}
}

\clearpage
\maketitlesupplementary
\appendix

\section{Additional FedRAN Algorithms and Proofs}
\label{app:fedran_theory}

This appendix provides helper algorithms and proofs for the statements in Section~\ref{sec:method}. All results are deterministic and condition on the feature matrix and label matrix used by FedRAN. When pseudo-labeling is enabled, the label matrix is the augmented label matrix produced by pseudo-labeling.

\subsection{Helper Algorithms}
\label{app:helper_algorithms}

This subsection details the three routines invoked by Algorithm~\ref{alg:fedran}. Algorithm~\ref{alg:client_svd} computes a client's rank-at-most-$r$ spectral summary: it factorizes $H_{k,t}$ through whichever of $H_{k,t}$ or $H_{k,t}^{\top}$ has the smaller leading dimension, so the thin SVD is always taken on the $\min\{\tilde n_{k,t},M\}$ side, and returns the top-$r_{k,t}$ right singular vectors $V_{k,t}$ with their singular values $\sigma_{k,t}$. Algorithm~\ref{alg:qr_svd_merge} is the server-side QR--SVD subspace merge used for both spatial (across clients) and temporal (across tasks) aggregation: given two summaries it concatenates their scaled bases, orthonormalizes the result by QR, re-diagonalizes the small $R$ factor by SVD, and truncates back to rank $r$; on the first merge an empty operand is returned unchanged (up to rank truncation). Algorithm~\ref{alg:pseudo_labeling} performs prototype-based pseudo-labeling: it forms $\ell_2$-normalized class prototypes from the labeled features, assigns each unlabeled feature to its nearest prototype by cosine similarity, and accepts the resulting pseudo-label only when the similarity exceeds the confidence threshold $\tau$.

\begin{algorithm}[t]
\caption{FedRAN: Federated Analytic Continual Learning}
\label{alg:fedran}
\small
\begin{algorithmic}[1]
\Require Clients $k\in\{1,\ldots,K\}$, task stream $\{\mathcal{T}_t\}_{t=1}^{T}$, frozen backbone $f(\cdot)$ with feature size $d$, shared Gaussian projection $P\in\mathbb{R}^{d\times M}$, retained rank $r$, ridge parameter $\lambda$, pseudo-labeling threshold $\tau$
\State Server initializes $(V_{1:0},\sigma_{1:0})\gets\emptyset$ and $B_{1:0}\gets\emptyset$
\For{$t=1,\ldots,T$}
    \State Let $C_t$ be the number of classes observed up to task $t$
    \Statex \textbf{Client-side local summarization}
    \For{each client $k$ in parallel}
        \State Extract frozen features for labeled and unlabeled data: $Z^{\ell}_{k,t}=f(X^{\ell}_{k,t})$, $Z^{u}_{k,t}=f(X^{u}_{k,t})$
        \If{prototype-based pseudo-labeling is enabled}
            \State $(\widetilde Z_{k,t},\widetilde y_{k,t})\gets\mathrm{PseudoLabel}(Z^{\ell}_{k,t},y^{\ell}_{k,t},Z^{u}_{k,t},\tau)$
        \Else
            \State $(\widetilde Z_{k,t},\widetilde y_{k,t})\gets(Z^{\ell}_{k,t},y^{\ell}_{k,t})$
        \EndIf
        \State Compute random features $H_{k,t}\gets\mathrm{ReLU}(\widetilde Z_{k,t}P)$, where $H_{k,t}\in\mathbb{R}^{\tilde n_{k,t}\times M}$
        \State Form one-hot labels $Y_{k,t}\in\mathbb{R}^{\tilde n_{k,t}\times C_t}$ from $\widetilde y_{k,t}$
        \State Compute truncated SVD summary $(V_{k,t},\sigma_{k,t})\gets\mathrm{ClientSVD}(H_{k,t},r)$
        \State Compute label-feature statistic $B_{k,t}\gets H_{k,t}^{\top}Y_{k,t}$
        \State Send $(V_{k,t},\sigma_{k,t},B_{k,t})$ to the server
    \EndFor

    \Statex \textbf{Server-side spatial--temporal aggregation}
    \State Initialize task spectral summary $(V_t,\sigma_t)\gets\emptyset$ and task label statistic $B_t\gets 0\in\mathbb{R}^{M\times C_t}$
    \For{each received summary $(V_{k,t},\sigma_{k,t},B_{k,t})$}
        \State $(V_t,\sigma_t)\gets\mathrm{MergeSVD}((V_t,\sigma_t),(V_{k,t},\sigma_{k,t}),r)$
        \State $B_t\gets B_t+B_{k,t}$
    \EndFor
    \State Pad $B_{1:t-1}$ with zero columns if needed, and set $B_{1:t}\gets B_{1:t-1}+B_t$
    \State $(V_{1:t},\sigma_{1:t})\gets\mathrm{MergeSVD}((V_{1:t-1},\sigma_{1:t-1}),(V_t,\sigma_t),r)$
    \State Set $\bar r_t\gets\mathrm{cols}(V_{1:t})$

    \Statex \textbf{Analytic classifier update}
    \State $\widetilde W_{1:t}\gets\left(\mathrm{diag}(\sigma_{1:t}^{2})+\lambda I_{\bar r_t}\right)^{-1}V_{1:t}^{\top}B_{1:t}$
    \State $W_{1:t}\gets V_{1:t}\widetilde W_{1:t}$
\EndFor
\State \Return Predictor $x\mapsto\arg\max_{c\in\{1,\ldots,C_T\}}\left[\mathrm{ReLU}(f(x)P)W_{1:T}\right]_c$
\end{algorithmic}
\end{algorithm}

\begin{algorithm}[h]
\caption{Client-Side Truncated SVD Summary}
\label{alg:client_svd}
\small
\begin{algorithmic}[1]
\Require Local random features $H_{k,t}\in\mathbb{R}^{\tilde n_{k,t}\times M}$, target rank $r$
\State Set $r_{k,t}\gets\min\{r,\tilde n_{k,t},M\}$
\If{$\tilde n_{k,t}\le M$}
    \State Compute thin SVD $H_{k,t}=U_{k,t}\mathrm{diag}(\sigma_{k,t})V_{k,t}^{\top}$
\Else
    \State Compute thin SVD $H_{k,t}^{\top}=V_{k,t}\mathrm{diag}(\sigma_{k,t})\widetilde U_{k,t}^{\top}$
\EndIf
\State Keep $V_{k,t}[:,1:r_{k,t}]\in\mathbb{R}^{M\times r_{k,t}}$ and $\sigma_{k,t}[1:r_{k,t}]\in\mathbb{R}^{r_{k,t}}$
\State \Return $(V_{k,t},\sigma_{k,t})$
\end{algorithmic}
\end{algorithm}

\begin{algorithm}[h]
\caption{Server-Side QR--SVD Subspace Merge}
\label{alg:qr_svd_merge}
\small
\begin{algorithmic}[1]
\Require Previous summary $(V_a,\sigma_a)$, new summary $(V_b,\sigma_b)$, target rank $r$
\If{$V_a$ is empty}
    \State \Return $(V_b[:,1:\min\{r,\mathrm{cols}(V_b)\}],\sigma_b[1:\min\{r,\mathrm{len}(\sigma_b)\}])$
\EndIf
\State Form $A\gets[V_a\mathrm{diag}(\sigma_a),\; V_b\mathrm{diag}(\sigma_b)]$
\State Compute QR factorization $A=QR$
\State Compute SVD $R=U_R\mathrm{diag}(\sigma)W_R^{\top}$
\State Set $r_{\mathrm{out}}\gets\min\{r,\mathrm{cols}(U_R)\}$
\State Keep up to the top $r$ components: $V\gets Q U_R[:,1:r_{\mathrm{out}}]$, $\sigma\gets\sigma[1:r_{\mathrm{out}}]$
\State \Return $(V,\sigma)$
\end{algorithmic}
\end{algorithm}

\begin{algorithm}[h]
\caption{Prototype-Based Pseudo-Labeling}
\label{alg:pseudo_labeling}
\small
\begin{algorithmic}[1]
\Require Labeled backbone features $Z^{\ell}_{k,t}$, labels $y^{\ell}_{k,t}$, unlabeled backbone features $Z^{u}_{k,t}$, confidence threshold $\tau$
\For{each class $c$ in $y^{\ell}_{k,t}$}
    \State $p_{k,t,c}\gets\mathrm{mean}(Z^{\ell}_{k,t}[y^{\ell}_{k,t}=c])$
    \State $p_{k,t,c}\gets p_{k,t,c}/\|p_{k,t,c}\|_2$
\EndFor
\State Initialize accepted pseudo-labeled set $(Z^p_{k,t},y^p_{k,t})\gets\emptyset$
\For{each unlabeled feature $u\in Z^u_{k,t}$}
    \State $\bar u\gets u/\|u\|_2$
    \State $c^{\star}\gets\arg\max_c \bar u^{\top}p_{k,t,c}$
    \If{$\bar u^{\top}p_{k,t,c^{\star}}\ge\tau$}
        \State Add $(u,c^{\star})$ to $(Z^p_{k,t},y^p_{k,t})$
    \EndIf
\EndFor
\State \Return $(\widetilde Z_{k,t},\widetilde y_{k,t})=([Z^{\ell}_{k,t};Z^p_{k,t}],\;[y^{\ell}_{k,t};y^p_{k,t}])$
\end{algorithmic}
\end{algorithm}

\subsection{Notation for the Proofs}
\label{app:proof_notation}

Let $G^{\star}_{1:t}$ denote the exact accumulated Gram statistic over all clients and tasks up to $t$, and let $B_{1:t}$ denote the corresponding label-feature statistic:
\begin{equation}
    G^{\star}_{1:t}
    =
    \sum_{\tau=1}^{t}\sum_{k=1}^{K}H_{k,\tau}^{\top}H_{k,\tau},
    \qquad
    B_{1:t}
    =
    \sum_{\tau=1}^{t}\sum_{k=1}^{K}H_{k,\tau}^{\top}Y_{k,\tau}.
\end{equation}
For a symmetric positive semidefinite matrix $S$, let $\mathcal{T}_r(S)$ denote its best rank-$r$ truncation obtained by keeping its $r$ largest eigenvalues and associated eigenvectors. Unless stated otherwise, singular values and eigenvalues are sorted in descending order and indexed from $1$. If an omitted singular value or eigenvalue does not exist, we set it to $0$ by convention.

For each local feature matrix $H_{k,\tau}$, let
\begin{equation}
    H_{k,\tau}
    =
    U_{k,\tau}
    \mathrm{diag}(s_{k,\tau,1},s_{k,\tau,2},\ldots)
    V_{k,\tau}^{\top}
\end{equation}
denote its full SVD, with $s_{k,\tau,1}\ge s_{k,\tau,2}\ge\cdots\ge0$. The vector $\sigma_{k,\tau}$ used in the main algorithm contains only the retained top singular values, while $s_{k,\tau,i}$ denotes the full spectrum used for analysis.

\subsection{Proof of Proposition~\ref{prop:qr_svd_covariance}}
\label{app:proof_qr_svd}

\begin{proof}
By construction,
\begin{equation}
    A=
    [V_a\mathrm{diag}(\sigma_a),\;V_b\mathrm{diag}(\sigma_b)].
\end{equation}
Therefore,
\begin{equation}
    AA^{\top}
    =
    V_a\mathrm{diag}(\sigma_a^2)V_a^{\top}
    +
    V_b\mathrm{diag}(\sigma_b^2)V_b^{\top}.
\end{equation}
This proves that the Gram matrix of the concatenated scaled basis equals the sum of the two input Gram sketches.

Now let $A=QR$ be a QR factorization with $Q^{\top}Q=I$, and let
\begin{equation}
    R=U_R\mathrm{diag}(\bar\sigma)W_R^{\top}
\end{equation}
be the SVD of $R$, with singular values $\bar\sigma_1\ge\bar\sigma_2\ge\cdots\ge0$. Then
\begin{equation}
    AA^{\top}
    =
    QRR^{\top}Q^{\top}
    =
    Q U_R\mathrm{diag}(\bar\sigma^2)U_R^{\top}Q^{\top}.
\end{equation}
Thus, without rank truncation, taking $V=QU_R$ and $\sigma=\bar\sigma$ gives
\begin{equation}
    V\mathrm{diag}(\sigma^2)V^{\top}=AA^\top.
\end{equation}

If FedRAN retains only the top $r$ components, then
\begin{equation}
    V=Q U_R^{(:,1:r)},
    \qquad
    \sigma=\bar\sigma_{1:r},
\end{equation}
and
\begin{equation}
    V\mathrm{diag}(\sigma^2)V^{\top}
    =
    \mathcal{T}_r(AA^\top).
\end{equation}
Since $AA^\top$ is symmetric positive semidefinite, the Eckart--Young--Mirsky theorem implies that $\mathcal{T}_r(AA^\top)$ is the best rank-$r$ approximation to $AA^\top$ in any unitarily invariant norm, including spectral and Frobenius norms. If $\lambda_i(AA^\top)=\bar\sigma_i^2$, then the residual satisfies
\begin{equation}
    \|AA^\top-\mathcal{T}_r(AA^\top)\|_2
    =
    \lambda_{r+1}(AA^\top),
\end{equation}
and
\begin{equation}
    \|AA^\top-\mathcal{T}_r(AA^\top)\|_F
    =
    \left(\sum_{i>r}\lambda_i(AA^\top)^2\right)^{1/2}.
\end{equation}
This proves the claim.
\end{proof}

\subsection{Proof of Theorem~\ref{thm:gram_error}}
\label{app:proof_gram_error}

\begin{proof}
We track exactly where information is discarded. The first source is local truncation at each client. For client $k$ and task $\tau$, the exact local Gram matrix is
\begin{equation}
    G_{k,\tau}=H_{k,\tau}^{\top}H_{k,\tau}.
\end{equation}
Let $\widetilde G_{k,\tau}$ be the rank-$r_{k,\tau}$ local SVD approximation used by FedRAN:
\begin{equation}
    \widetilde G_{k,\tau}
    =
    V_{k,\tau}
    \mathrm{diag}(\sigma_{k,\tau}^{2})
    V_{k,\tau}^{\top}.
\end{equation}
Because $\widetilde G_{k,\tau}$ keeps the top $r_{k,\tau}$ eigen-directions of $G_{k,\tau}$, the local residual
\begin{equation}
    R^{\mathrm{loc}}_{k,\tau}
    =
    G_{k,\tau}-\widetilde G_{k,\tau}
\end{equation}
is positive semidefinite. Its spectral, Frobenius, and trace errors are
\begin{equation}
    \|R^{\mathrm{loc}}_{k,\tau}\|_2
    =
    s_{k,\tau,r_{k,\tau}+1}^{2},
    \label{eq:local_spectral_residual}
\end{equation}
\begin{equation}
\begin{aligned}
    \|R^{\mathrm{loc}}_{k,\tau}\|_F
    &=
    \left(\sum_{i>r_{k,\tau}}s_{k,\tau,i}^{4}\right)^{1/2},\\
    \mathrm{tr}(R^{\mathrm{loc}}_{k,\tau})
    &=
    \sum_{i>r_{k,\tau}}s_{k,\tau,i}^{2}.
\end{aligned}
    \label{eq:local_energy_residual}
\end{equation}
The index $r_{k,\tau}+1$ denotes the first omitted singular value of the full matrix $H_{k,\tau}$, not an entry of the retained vector $\sigma_{k,\tau}$.

The second source of approximation is server-side merge truncation. Let $\mathcal{M}_t$ denote the set of all spatial and temporal QR--SVD merges performed up to task $t$. For each merge $j\in\mathcal{M}_t$, let $A_j$ be the concatenated scaled basis used in Eq.~\eqref{eq:merge_A}, and define
\begin{equation}
    S_j=A_jA_j^\top.
\end{equation}
Before truncation, $S_j$ is exactly the sum of the two input Gram sketches by Proposition~\ref{prop:qr_svd_covariance}. FedRAN replaces $S_j$ with $\mathcal{T}_r(S_j)$, introducing the residual
\begin{equation}
    R^{\mathrm{merge}}_j
    =
    S_j-\mathcal{T}_r(S_j).
\end{equation}
If $\lambda_{j,1}\ge\lambda_{j,2}\ge\cdots\ge0$ are the eigenvalues of $S_j$, then
\begin{equation}
    \|R^{\mathrm{merge}}_j\|_2
    =
    \lambda_{j,r+1},
    \label{eq:merge_spectral_residual}
\end{equation}
\begin{equation}
    \|R^{\mathrm{merge}}_j\|_F
    =
    \left(\sum_{i>r}\lambda_{j,i}^{2}\right)^{1/2},
    \qquad
    \mathrm{tr}(R^{\mathrm{merge}}_j)
    =
    \sum_{i>r}\lambda_{j,i}.
    \label{eq:merge_energy_residual}
\end{equation}

We now relate these discarded terms to the final FedRAN sketch. Before any server-side merge, replacing each exact local Gram matrix by its local sketch discards $\sum_{\tau,k}R^{\mathrm{loc}}_{k,\tau}$. During each QR--SVD merge, replacing $S_j$ by $\mathcal{T}_r(S_j)$ discards $R^{\mathrm{merge}}_j$. Since all these residuals lie in the same ambient $M$-dimensional random-feature space, the final error decomposes as
\begin{equation}
    G^{\star}_{1:t}-\widetilde G_{1:t}
    =
    \sum_{\tau=1}^{t}\sum_{k=1}^{K}
    R^{\mathrm{loc}}_{k,\tau}
    +
    \sum_{j\in\mathcal{M}_t}
    R^{\mathrm{merge}}_j.
    \label{eq:gram_error_decomposition}
\end{equation}
Taking spectral norms and using the residual identities above gives
\begin{equation}
    \|G^{\star}_{1:t}-\widetilde G_{1:t}\|_2
    \le
    \sum_{\tau=1}^{t}\sum_{k=1}^{K}
    s_{k,\tau,r_{k,\tau}+1}^{2}
    +
    \sum_{j\in\mathcal{M}_t}
    \lambda_{j,r+1}.
\end{equation}
This proves Eq.~\eqref{eq:gram_bound}.
\end{proof}

\subsection{An Interpretable Relative-Tail Bound}
\label{app:relative_tail_bound}

Theorem~\ref{thm:gram_error} gives an instance-dependent bound in terms of the exact discarded singular and eigenvalues. We can further express this bound through aggregate feature energy under a standard relative-tail assumption.

\noindent\textbf{Assumption.}
Suppose there exist $\eta_{\mathrm{loc}},\eta_{\mathrm{merge}}\in[0,1]$ such that every local truncation discards at most an $\eta_{\mathrm{loc}}$ fraction of local feature energy,
\begin{equation}
    \sum_{i>r_{k,\tau}}s_{k,\tau,i}^{2}
    \le
    \eta_{\mathrm{loc}}\|H_{k,\tau}\|_F^2,
    \qquad
    \forall k,\tau,
    \label{eq:relative_local_tail}
\end{equation}
and every server-side merge truncation discards at most an $\eta_{\mathrm{merge}}$ fraction of the trace energy being merged,
\begin{equation}
    \sum_{i>r}\lambda_{j,i}
    \le
    \eta_{\mathrm{merge}}\mathrm{tr}(S_j),
    \qquad
    \forall j\in\mathcal{M}_t.
    \label{eq:relative_merge_tail}
\end{equation}
Let
\begin{equation}
    E_t
    =
    \sum_{\tau=1}^{t}\sum_{k=1}^{K}
    \|H_{k,\tau}\|_F^2
\end{equation}
be the total random-feature energy observed up to task $t$. Let $L_t$ be the maximum number of rank-truncating server merges in which any local client summary can participate before task $t$.

\begin{fedrancor}[Relative-tail Gram bound]
\label{cor:relative_tail_gram}
Under the relative-tail assumption,
\begin{equation}
    \|G^{\star}_{1:t}-\widetilde G_{1:t}\|_2
    \le
    \mathrm{tr}(G^{\star}_{1:t}-\widetilde G_{1:t})
    \le
    \left(\eta_{\mathrm{loc}}+\eta_{\mathrm{merge}}L_t\right)E_t.
    \label{eq:relative_tail_gram_bound}
\end{equation}
If additionally $\|h(x)\|_2^2\le R^2$ for every random feature vector and $N_t=\sum_{\tau=1}^{t}\sum_{k=1}^{K}\tilde n_{k,\tau}$, then
\begin{equation}
    \|G^{\star}_{1:t}-\widetilde G_{1:t}\|_2
    \le
    \left(\eta_{\mathrm{loc}}+\eta_{\mathrm{merge}}L_t\right)R^2N_t.
    \label{eq:relative_tail_data_bound}
\end{equation}
\end{fedrancor}

\begin{proof}
From Eq.~\eqref{eq:gram_error_decomposition}, the error is a sum of positive semidefinite residuals. Hence its spectral norm is at most its trace. The local residual trace is $\sum_{i>r_{k,\tau}}s_{k,\tau,i}^2$, which is bounded by Eq.~\eqref{eq:relative_local_tail}. The merge residual trace is $\sum_{i>r}\lambda_{j,i}$, which is bounded by Eq.~\eqref{eq:relative_merge_tail}. Since each local summary can participate in at most $L_t$ rank-truncating server merges, the total trace energy exposed to merge truncation is at most $L_tE_t$. Therefore,
\begin{equation}
    \mathrm{tr}(G^{\star}_{1:t}-\widetilde G_{1:t})
    \le
    \eta_{\mathrm{loc}}E_t+\eta_{\mathrm{merge}}L_tE_t.
\end{equation}
The final bound follows from $E_t=\sum_{\tau,k}\|H_{k,\tau}\|_F^2\le R^2N_t$ when every random feature vector has squared norm at most $R^2$.
\end{proof}

\subsection{Proof of Theorem~\ref{thm:weight_error}}
\label{app:proof_weight_error}

\begin{proof}
We suppress the subscript $1{:}t$ for readability and let $\bar r=\mathrm{cols}(V)$. Let
\begin{equation}
    W^{\star}=(G^{\star}+\lambda I_M)^{-1}B
\end{equation}
be the full ridge solution. Let
\begin{equation}
    \widetilde G=V\Lambda V^{\top},
    \qquad
    \Lambda=\mathrm{diag}(\sigma^2),
\end{equation}
be the FedRAN Gram sketch. Define the full ridge solution associated with the sketched Gram matrix:
\begin{equation}
    W_{\widetilde G}=(\widetilde G+\lambda I_M)^{-1}B.
\end{equation}
Then
\begin{equation}
    \|W^{\star}-W\|_F
    \le
    \|W^{\star}-W_{\widetilde G}\|_F
    +
    \|W_{\widetilde G}-W\|_F.
    \label{eq:weight_error_decomp}
\end{equation}

For the first term, we use the standard inverse-difference identity
\begin{equation}
\begin{aligned}
    &(G^{\star}+\lambda I_M)^{-1}
    -
    (\widetilde G+\lambda I_M)^{-1}\\
    &\qquad=
    (G^{\star}+\lambda I_M)^{-1}
    (\widetilde G-G^{\star})
    (\widetilde G+\lambda I_M)^{-1}.
\end{aligned}
\end{equation}
Since $G^{\star}$ and $\widetilde G$ are positive semidefinite, both inverse factors have spectral norm at most $1/\lambda$. Therefore,
\begin{equation}
    \|W^{\star}-W_{\widetilde G}\|_F
    \le
    \frac{\|G^{\star}-\widetilde G\|_2}{\lambda^2}\|B\|_F
    \le
    \frac{\varepsilon_G(t)}{\lambda^2}\|B\|_F.
    \label{eq:sketch_ridge_error}
\end{equation}

For the second term, because $V^{\top}V=I_{\bar r}$,
\begin{equation}
    \widetilde G+\lambda I_M
    =
    V(\Lambda+\lambda I_{\bar r})V^{\top}
    +
    \lambda(I_M-VV^{\top}).
\end{equation}
Hence,
\begin{equation}
    (\widetilde G+\lambda I_M)^{-1}
    =
    V(\Lambda+\lambda I_{\bar r})^{-1}V^{\top}
    +
    \frac{1}{\lambda}(I_M-VV^{\top}).
    \label{eq:low_rank_ridge_inverse}
\end{equation}
FedRAN uses the subspace-constrained component
\begin{equation}
    W
    =
    V(\Lambda+\lambda I_{\bar r})^{-1}V^{\top}B.
\end{equation}
Substituting Eq.~\eqref{eq:low_rank_ridge_inverse} gives
\begin{equation}
    W_{\widetilde G}-W
    =
    \frac{1}{\lambda}(I_M-VV^{\top})B,
\end{equation}
and therefore
\begin{equation}
    \|W_{\widetilde G}-W\|_F
    =
    \frac{1}{\lambda}
    \|(I_M-VV^{\top})B\|_F.
    \label{eq:subspace_residual_error}
\end{equation}
Combining Eq.~\eqref{eq:weight_error_decomp}, Eq.~\eqref{eq:sketch_ridge_error}, and Eq.~\eqref{eq:subspace_residual_error} proves the theorem.
\end{proof}

\subsection{Proof of Corollary~\ref{cor:score_margin}}
\label{app:proof_score_margin}

\begin{proof}
Let $\Delta W=W^{\star}_{1:t}-W_{1:t}$. Then
\begin{equation}
    \|s^{\star}-s\|_2
    =
    \|h^{\top}\Delta W\|_2
    \le
    \|h\|_2\|\Delta W\|_F
    \le
    \|h\|_2\varepsilon_W(t),
\end{equation}
which proves the score perturbation bound.

Let $y^{\star}=\arg\max_c s^{\star}_c$ and let $\delta=\|h\|_2\varepsilon_W(t)$. The score perturbation bound implies $|s_c-s^{\star}_c|\le\delta$ for every class $c$. Thus,
\begin{equation}
    s_{y^{\star}}\ge s^{\star}_{y^{\star}}-\delta,
    \qquad
    s_c\le s^{\star}_c+\delta
    \quad\text{for all }c\ne y^{\star}.
\end{equation}
If
\begin{equation}
    s^{\star}_{y^{\star}}-\max_{c\ne y^{\star}}s^{\star}_c>2\delta,
\end{equation}
then $s_{y^{\star}}>s_c$ for all $c\ne y^{\star}$, so FedRAN and the full ridge classifier predict the same class.
\end{proof}

\subsection{Remarks on the Bounds}
\label{app:bound_remarks}

\noindent\textbf{Local versus merge approximation.}
Theorem~\ref{thm:gram_error} separates the approximation error into two interpretable sources. Local errors arise because clients retain only the top singular directions of $H_{k,t}$. Merge errors arise because the server repeatedly truncates merged spectral summaries back to rank $r$. Increasing $r$ decreases both terms but increases communication and server-side computation.

\noindent\textbf{Subspace residual.}
The second term in Theorem~\ref{thm:weight_error}, $\|(I-VV^{\top})B\|_F/\lambda$, appears because FedRAN intentionally solves the classifier in the retained subspace. If label-feature statistics are mostly contained in this subspace, FedRAN closely matches the full ridge solution. This term also clarifies the role of the retained-rank budget: a larger retained subspace can reduce the residual at the cost of larger communication.

\noindent\textbf{Pseudo-labeling.}
The theory above conditions on the label matrix used in the analytic update. If pseudo-labeling introduces label noise, then $B$ changes because it is computed from the augmented labels. A separate label-noise analysis can be layered on top of these deterministic bounds by controlling $\|B_{\mathrm{pseudo}}-B_{\mathrm{true}}\|_F$, but this is orthogonal to the spectral aggregation analysis.

\begin{table*}[t]
\centering
\caption{Dominant computation cost. Client-side cost is per client per task; server-side cost is per task.}
\label{tab:comp_complexity_appendix}
\small
\begin{tabular}{lcc}
\toprule
Method & Client-side cost & Server-side cost \\
\midrule
Iterative FCL
& $O(R_tE\tilde n_{k,t}C_{\mathrm{bp}})$
& $O(R_tKS_{\mathrm{model}})$ \\
Exact analytic FCL
& $O(\tilde n_{k,t}C_f+\tilde n_{k,t}dM+\tilde n_{k,t}M^2+\tilde n_{k,t}M)$
& $O(K(M^2+MC_t)+M^3+M^2C_t)$ \\
FedRAN
& $O\!\left(\tilde n_{k,t}C_f+\tilde n_{k,t}dM+\tilde n_{k,t}M+\min\{\tilde n_{k,t}^{2}M,\tilde n_{k,t}M^{2}\}\right)$
& $O(K(Mr^2+r^3+MC_t)+MrC_t)$ \\
\bottomrule
\end{tabular}
\end{table*}

\section{Additional Evaluation Details}
\label{app:fedran_eval}

This appendix provides additional metrics, baselines, dataset details, and cost analyses for the evaluation in Section~\ref{evaluation}.

\subsection{Accuracy Metrics}
\label{subsec:prelim_accuracy}

Let $\psi_t$ denote the global predictor after completing task $t$, and let $\mathcal{E}_{1:t}$ denote the test set over all classes observed up to task $t$. The accuracy after task $t$ is $A_t=\frac{1}{|\mathcal{E}_{1:t}|}\sum_{(x,y)\in\mathcal{E}_{1:t}}\mathbf{1}\{\psi_t(x)=y\}$. We report the final accuracy $A_T$ and the average accuracy $A_{\mathrm{avg}}=\frac{1}{T}\sum_{t=1}^{T} A_t$, which measures performance across the full task stream.

\subsection{Efficiency Metrics}
\label{subsec:prelim_resources}

We characterize each FCL algorithm $\mathcal{A}$ using communication, computation, label rate, and feature drift.

\noindent\textbf{Communication.}
Communication is the maximum upload by any client for a single task, accumulated over all rounds used for that task: $\mathrm{Comm}(\mathcal{A})=\max_{k,t}\,\mathrm{bytes}^{\mathrm{upload}}_{k,t}(\mathcal{A})$.

\noindent\textbf{Computation.}
Computation is measured by the wall-clock time required to incorporate a task. Let
$T^{\mathrm{client}}_{k,t}(\mathcal{A})$ be the local model update time for client $k$ at task $t$, and let
$T^{\mathrm{server}}_{t}(\mathcal{A})$ be the corresponding server-side aggregation and update time. We define
\begin{equation}
    \mathrm{Time}(\mathcal{A})
    =
    \frac{1}{T}
    \sum_{t=1}^{T}
    \left(
    \frac{1}{K}
    \sum_{k=1}^{K}
    T^{\mathrm{client}}_{k,t}(\mathcal{A})
    +
    T^{\mathrm{server}}_{t}(\mathcal{A})
    \right).
    \label{eq:time_metric}
\end{equation}

\noindent\textbf{Label rate.}
For label-scarce streams, the local label rate is $\rho_{k,t}=|\mathcal{D}^{\ell}_{k,t}|/(|\mathcal{D}^{\ell}_{k,t}|+|\mathcal{D}^{u}_{k,t}|)$; a smaller $\rho_{k,t}$ means that fewer local samples have ground-truth labels at task $t$.

\noindent\textbf{Feature drift.}
For methods that update a shared feature extractor, feature drift measures how much the representation of a fixed reference set changes after later updates~\cite{venkatesha2022addressing}. Let $g^{(s)}(\cdot)$ be the feature extractor after update stage $s$, and let $g^{(0)}(\cdot)$ be the reference extractor. For a reference set $\mathcal{R}_{k,\tau}$ from client $k$ and task $\tau$, we define
\begin{equation}
    \mathrm{Drift}_{k,\tau}(s)
    =
    \frac{1}{|\mathcal{R}_{k,\tau}|}
    \sum_{x\in\mathcal{R}_{k,\tau}}
    \|g^{(s)}(x)-g^{(0)}(x)\|_2.
    \label{eq:feature_drift}
\end{equation}
For methods with a fixed feature extractor, this quantity is zero.

\subsection{Datasets}
\label{app:dataset}

Table~\ref{tab:datasets} summarizes the datasets and task splits used in our evaluation.

\begin{table}
\centering
\caption{Datasets used for evaluation.}
\label{tab:datasets}
\footnotesize
\setlength{\tabcolsep}{4pt}
\begin{tabular}{lccc}
\toprule
Dataset & Type & Classes & Tasks \\
\midrule
CIFAR-100~\cite{cifar100} 
& Standard image classification & 100 & 5/10 \\
ImageNet-R~\cite{imagenetR} 
& OOD generalization & 200 & 10 \\
VTAB~\cite{vtab} 
& Diverse visual domains & 50 & 5 \\
\bottomrule
\end{tabular}
\end{table}

\subsection{Baselines}
\label{app:baselines}
We compare FedRAN against representative baselines for federated continual learning. Finetune sequentially trains client models on new tasks and aggregates them using FedAvg~\cite{FedAvg}. FedEWC~\cite{fedEWC}, FedLwF~\cite{fedLWF}, and FediCaRL~\cite{icarl} adapt standard continual-learning strategies to the federated setting: FedEWC regularizes changes to weights important from previous tasks, FedLwF uses distillation to preserve previous-task behavior, and FediCaRL is a federated adaptation of iCaRL's exemplar-replay strategy. TARGET~\cite{target} is an exemplar-free federated class-continual learning method based on distillation and synthetic data generation.
STSA-E estimates the Gram statistic from class-wise feature sums and counts to reduce communication~\cite{stsa}.

For pretrained ViT-based comparisons, we also include prompt- and adapter-based continual learning methods. DualPrompt~\cite{dualIP} learns complementary prompts for rehearsal-free continual learning. CodaPrompt~\cite{CodaP} uses decomposed attention-based prompts for continual adaptation. Fed-CPrompt~\cite{FedCPrompt} extends prompt learning to rehearsal-free federated continual learning using contrastive prompts. PiLoRA~\cite{Pilora} uses prototype-guided LoRA adaptation for federated class-incremental learning.

\subsection{Communication and Computation Cost}
\label{app:cost_analysis}

We summarize the per-task cost of FedRAN and compare it with iterative optimization-based FCL and exact analytic FCL. The costs below exclude the one-time distribution of the frozen backbone and shared Gaussian random projection. We assume that after each task the server returns a dense classifier $W_{1:t}\in\mathbb{R}^{M\times C_t}$ to each client for local inference. Thus, analytic methods have downlink cost $MC_t$ values per client per task. If inference is performed only on the server, this downlink can be omitted.

Let $K$ be the number of clients, $R_t$ the number of communication rounds for task $t$, $E$ the number of local epochs, $S_{\mathrm{model}}$ the number of trainable model parameters, $C_{\mathrm{bp}}$ the cost of one forward/backward pass per sample through the trainable model, $C_f$ the cost of one frozen-backbone forward pass per sample, $\tilde n_{k,t}$ the number of samples used by client $k$ at task $t$ after pseudo-labeling, $d$ the frozen-backbone feature dimension, $M$ the random-feature dimension, $r$ the retained rank, and $C_t$ the number of classes observed up to task $t$.

\noindent\textbf{Communication.}
Table~\ref{tab:comm_complexity_appendix} reports per-client per-task communication in transmitted floating-point values. Iterative FCL exchanges model parameters or parameter updates over multiple communication rounds. Exact analytic FCL avoids iterative model exchange, but each client uploads the full local Gram matrix $G_{k,t}\in\mathbb{R}^{M\times M}$ and the label-feature statistic $B_{k,t}\in\mathbb{R}^{M\times C_t}$. FedRAN instead uploads a low-rank spectral summary $(V_{k,t},\sigma_{k,t})$ and the exact label-feature statistic $B_{k,t}$, reducing the dominant uplink term from $M^2$ to $Mr+r$.

\begin{table}[t]
\centering
\caption{Per-client per-task communication cost in transmitted floating-point values. The downlink assumes the server returns the dense classifier $W_{1:t}\in\mathbb{R}^{M\times C_t}$ to each client.}
\label{tab:comm_complexity_appendix}
\small
\begin{tabular}{lcc}
\toprule
Method & Uplink & Downlink \\
\midrule
Iterative FCL
& $R_tS_{\mathrm{model}}$
& $R_tS_{\mathrm{model}}$ \\
Exact analytic FCL
& $M^2+MC_t$
& $MC_t$ \\
FedRAN
& $Mr+r+MC_t$
& $MC_t$ \\
\bottomrule
\end{tabular}
\end{table}

\noindent\textbf{Computation.}
Table~\ref{tab:comp_complexity_appendix} reports dominant computation costs. Client-side cost is per client per task, while server-side cost is per task. For iterative FCL, client computation is dominated by repeated local backpropagation over $R_t$ communication rounds and $E$ local epochs, while server computation is dominated by aggregating model parameters. Exact analytic FCL avoids backpropagation but constructs the full local Gram matrix on the client and solves a full $M$-dimensional ridge system on the server. FedRAN avoids both full Gram construction and full Gram inversion. Each client computes random features, constructs $B_{k,t}$ by sparse class-wise accumulation, and computes a truncated SVD summary of $H_{k,t}$. The server performs QR--SVD merges over rank-$r$ summaries and solves the classifier in the retained subspace.

For the client-side SVD, an exact thin SVD of $H_{k,t}\in\mathbb{R}^{\tilde n_{k,t}\times M}$ costs
\[
    O\!\left(\min\{\tilde n_{k,t}^{2}M,\;\tilde n_{k,t}M^{2}\}\right),
\]
depending on whether the computation is performed through $H_{k,t}$ or $H_{k,t}^{\top}$. In the common regime $\tilde n_{k,t}\ll M$, this becomes $O(\tilde n_{k,t}^{2}M)$.


\end{document}